\documentclass[sigconf]{acmart}

\usepackage{microtype}
\usepackage{listings}

\usepackage{hyperref}

\usepackage{mathtools}
\usepackage{amsthm}

\usepackage[capitalize,noabbrev]{cleveref}

\theoremstyle{plain}
\newtheorem{theorem}{Theorem}[section]

\newtheorem{corollary}[theorem]{Corollary}
\theoremstyle{definition}

\theoremstyle{remark}

\usepackage{algorithm}
\usepackage{algorithmic}
\usepackage{listings}
\usepackage{xspace}
\usepackage{enumitem}
\usepackage{bm}
\usepackage{multirow}
\usepackage{subcaption}
\usepackage{balance}

\DeclareMathOperator*{\argmin}{arg\,min}
\DeclareMathOperator*{\argmax}{arg\,max}

\newcommand{\diag}{\text{diag}\xspace}

\AtBeginDocument{%
	}

\copyrightyear{2023} 
\acmYear{2023} 
\setcopyright{acmlicensed}\acmConference[KDD '23]{Proceedings of the 29th ACM SIGKDD Conference on Knowledge Discovery and Data Mining}{August 6--10, 2023}{Long Beach, CA, USA}
\acmBooktitle{Proceedings of the 29th ACM SIGKDD Conference on Knowledge Discovery and Data Mining (KDD '23), August 6--10, 2023, Long Beach, CA, USA}
\acmPrice{15.00}
\acmDOI{10.1145/3580305.3599501}
\acmISBN{979-8-4007-0103-0/23/08}

\begin{document}

	\title{Sharpness-Aware Minimization Revisited: Weighted Sharpness as a Regularization Term}

	\author{Yun Yue}
	\authornote{Co-first authors with equal contributions.}
	\orcid{0000-0002-1430-7519}
	\affiliation{%
		\institution{Ant Group}
		\streetaddress{569 Xixi Rd}
		\city{Hangzhou}
		\state{Zhejiang}
		\country{China}
	}
	\email{yueyun.yy@antgroup.com}

	\author{Jiadi Jiang}
	\orcid{0000-0002-5998-0037}
	\authornotemark[1]
	\affiliation{%
		\institution{Ant Group}
		\streetaddress{569 Xixi Rd}
		\city{Hangzhou}
		\state{Zhejiang}
		\country{China}
	}
	\email{jiadi.jjd@antgroup.com}
	
	\author{Zhiling Ye}
	\orcid{0009-0001-1229-5876}
	\authornotemark[1]
	\affiliation{%
		\institution{Ant Group}
		\streetaddress{569 Xixi Rd}
		\city{Hangzhou}
		\state{Zhejiang}
		\country{China}
	}
	\email{yezhiling.yzl@antgroup.com}
	
	\author{Ning Gao}
	\orcid{0000-0002-4458-2132}
	\affiliation{%
		\institution{Ant Group}
		\streetaddress{East Tower Beijing World Finacial Centre, NO.1 East 3rd Ring Middle Road}
		\city{Beijing}
		\country{China}
	}
	\email{yunsheng.gn@antgroup.com}
	
	\author{Yongchao Liu}
	\orcid{0000-0003-3440-9675}
	\affiliation{%
		\institution{Ant Group}
		\streetaddress{569 Xixi Rd}
		\city{Hangzhou}
		\state{Zhejiang}
		\country{China}
	}
	\email{yongchao.ly@antgroup.com}
	
	\author{Ke Zhang}
	\orcid{0009-0008-6685-1293}
	\affiliation{%
		\institution{Ant Group}
		\streetaddress{East Tower Beijing World Finacial Centre, NO.1 East 3rd Ring Middle Road}
		\city{Beijing}
		\country{China}
	}
	\email{yingzi.zk@antgroup.com}

	\renewcommand{\shortauthors}{Yun Yue et al.}

	\begin{abstract}
		Deep Neural Networks (DNNs) generalization is known to be closely related to the flatness of minima, leading to the development of Sharpness-Aware Minimization (SAM) for seeking flatter minima and better generalization. In this paper, we revisit the loss of SAM and propose a more general method, called WSAM, by incorporating sharpness as a regularization term. We prove its generalization bound through the combination of PAC and Bayes-PAC techniques, and evaluate its performance on various public datasets. The results demonstrate that WSAM achieves improved generalization, or is at least highly competitive, compared to the vanilla optimizer, SAM and its variants. The code is available at this link\footnote{\url{https://github.com/intelligent-machine-learning/dlrover/tree/master/atorch/atorch/optimizers}}.
	\end{abstract}



	\maketitle

	\graphicspath{{figures}}

	\section{Introduction}
With the development of deep learning, DNNs with high over-parameterization have achieved tremendous success in various machine learning scenarios such as CV and NLP\footnote{\url{https://paperswithcode.com/sota/}}. Although the over-parameterized models are prone to overfit the training data \cite{rethink_generalization}, they do generalize well most of the time. The mystery of generalization has received increasing attention and has become a hot research topic in deep learning.

Recent works have revealed that the generalization is closely related to the flatness of minima, i.e., the flatter minima of the loss landscape could achieve lower generalization error \cite{flat_minima, generalization_gap, generalization_bound, entropy_sgd, swa, visualize_loss, understand_generalization}.
Sharpness-Aware Minimization (SAM) \cite{sam} is one of the most promising methods for finding flatter minima to improve generalization.
It is widely used in various fields, such as CV \cite{vit_sam}, NLP \cite{lm_sam} and bi-level learning \cite{maml_sam}, and has significantly outperformed the state-of-the-art (SOTA) method in those fields.

For the exploration of the flatter minima, SAM defines the sharpness of the loss function $L$ at $\bm{w}$ as follows:
\begin{equation}
	\tilde{L}(\bm{w}) := \underbrace{\max_{\|\bm{\delta}\|\leq\rho}L(\bm{w} + \bm{\delta})}_{L^{SAM}(\bm{w})} - L(\bm{w}).
	\label{eq:loss_sharp}
\end{equation}
\citet{gsam} proves that $\tilde{L}(\bm{w})$ is an approximation to the dominant eigenvalue of the Hessian at local minima, implying that $\tilde{L}(\bm{w})$ is indeed an effective metric of the sharpness. However, $\tilde{L}(\bm{w})$ can only be used to find flatter areas but not minima, which could potentially lead to convergence at a point where the loss is still large. Thus, SAM adopts $\tilde{L}(\bm{w}) + L(\bm{w})$, i.e., $L^{SAM}(\bm{w})$, to be the loss function. It can be thought of as a compromise between finding the flatter surface and the smaller minima by giving the same weights to $\tilde{L}(\bm{w})$ and $L(\bm{w})$.

In this paper, we rethink the construction of $L^{SAM}(\bm{w})$ and regard $\tilde{L}(\bm{w})$ as a regularization term. We develop a more general and effective algorithm, called WSAM (\textbf{W}eighted \textbf{S}harpness-\textbf{A}ware \textbf{M}inimization), whose loss function is regularized by a weighted sharpness term $\frac{\gamma}{1-\gamma}\tilde{L}(\bm{w})$, where the hyperparameter $\gamma$ controls the weight of sharpness. In Section~\ref{sec:4} we demonstrate how $\gamma$ directs the loss trajectory to find either flatter or lower minima. Our contribution can be summarized as follows.
\begin{itemize}[leftmargin=10pt]
    \item We propose WSAM, which regards the sharpness as a regularization and assigns different weights across different tasks. Inspired by \citet{weight_decouple}, we put forward a \text{``weight decouple''} technique to address the regularization in the final updated formula, aiming to reflect only the current step's sharpness. When the base optimizer is not simply \textsc{SGD} \cite{sgd}, such as \textsc{SgdM} \cite{momentum2} and \textsc{Adam} \cite{adam}, WSAM has significant differences in form compared to SAM. The ablation study shows that this technique can improve performance in most cases.
    \item We establish theoretically sound convergence guarantees in both convex and non-convex stochastic settings, and give a generalization bound by mixing PAC and Bayes-PAC techniques.
    \item We validate WSAM on a wide range of common tasks on public datasets. Experimental results show that WSAM yields better or highly competitive generalization performance versus SAM and its variants.
\end{itemize}

	\section{Related Work}
Several SAM variants have been developed to improve either effectiveness or efficiency. GSAM \cite{gsam} minimizes both $L^{SAM}(\bm{w})$ and $\tilde{L}(\bm{w})$ of Eq.~\eqref{eq:loss_sharp} simultaneously by employing the gradient projection technique. Compared to SAM, GSAM keeps $L^{SAM}(\bm{w})$ unchanged and decreases the surrogate gap, i.e. $\tilde{L}(\bm{w})$, by increasing $L(\bm{w})$. In other words, it gives more weights to $\tilde{L}(\bm{w})$ than $L(\bm{w})$ implicitly. ESAM \cite{esam} improves the efficiency of SAM without sacrificing accuracy by selectively applying SAM update with stochastic weight perturbation and sharpness-sensitivity data selection.

ASAM \cite{asam} and Fisher SAM \cite{fsam} try to improve the geometrical structure of the exploration area of $L^{SAM}(\bm{w})$. ASAM introduces adaptive sharpness, which normalizes the radius of the exploration region, i.e., replacing $\|\bm{\delta}\|$ of Eq.~\eqref{eq:loss_sharp} with $\|\bm{\delta} / \bm{w}\|$, to avoid the scale-dependent problem that SAM can suffer from. Fisher SAM employs another replacement by using $\sqrt{\|\bm{\delta}^T\diag(F)\bm{\delta}\|}$ as an intrinsic metric that can depict the underlying statistical manifold more accurately, where $F$ is the empirical Fisher information.

	\section{Preliminary}
\subsection{Notation}
We use lowercase letters to denote scalars, boldface lowercase letters to denote vectors, and uppercase letters to denote matrices. We denote a sequence of vectors by subscripts, that is, $\bm{x}_1, \dots, \bm{x}_t$ where $t \in [T] := \left\{1, 2, \dots, T\right\}$, and entries of each vector by an additional subscript, e.g., $x_{t, i}$. For any vectors $\bm{x}, \bm{y}\in\mathbb{R}^n$, we write $\bm{x}^{T}\bm{y}$ or $\bm{x}\cdot \bm{y}$ for the standard inner product, $\bm{x}\bm{y}$ for element-wise multiplication, $\bm{x} / \bm{y}$ for element-wise division, $\sqrt{\bm{x}}$ for element-wise square root, $\bm{x}^2$ for element-wise square. For the standard Euclidean norm, $\|\bm{x}\| = \|\bm{x}\|_2 = \sqrt{\left<\bm{x}, \bm{x}\right>}$. We also use $\|\bm{x}\|_{\infty} = \max_{i}|x^{(i)}|$ to denote $\ell_{\infty}$-norm, where $x^{(i)}$ is the $i$-th element of $\bm{x}$.

Let $\mathcal{S}_m$ be a training set of size $m$, i.e., $\mathcal{S}_m = \{(\bm{x}_i, y_i)\}_{i=1,\dots, m}$, where $\bm{x}_i\in\mathcal{X} \subseteq\mathbb{R}^k$ is an instance and $y_i\in\mathcal{Y}$ is a label. Denote the hypotheses space $\mathcal{H} = \{h_{\bm{w}}: \bm{w}\in\mathbb{R}^n\}$, where $h_{\bm{w}}(\cdot) : \mathcal{X}\rightarrow\mathcal{Y}$ is a hypothesis. Denote the training loss
\begin{equation*}
	L(\bm{w}) := \frac{1}{m}\sum_{k=1}^{m}\ell(h_{\bm{w}}(\bm{x}_k), y_k),
\end{equation*}
where $\ell(h_{\bm{w}}(\bm{x}), y)$ (we will often write $\ell(\bm{w})$ for simplicity) is a loss function measuring the performance of the parameter $\bm{w}$ on the example $(\bm{x}, y)$. Since it is inefficient to calculate the exact gradient in each optimization iteration when $m$ is large, we usually adopt a stochastic gradient with mini-batch, which is
\begin{equation*}
	g(\bm{w}) = \frac{1}{|\mathcal{B}|} \sum_{k\in\mathcal{B}} \nabla \ell(h_{\bm{w}}(\bm{x}_k), y_k),
\end{equation*}
where $\mathcal{B} \subset \{1, \dots, m\}$ is the sample set of size $|\mathcal{B}| \ll m$. Furthermore, let $\ell_t(\bm{w})$ be the loss function of the model at $t$-step.

\subsection{Sharpness-Aware Minimization}
SAM is a min-max optimization problem of solving $L^{SAM}(\bm{w})$ defined in Eq.~\eqref{eq:loss_sharp}.
First, SAM approximates the inner maximization problem using a first-order Taylor expansion around $\bm{w}$, i.e., 
\begin{equation*}
	\begin{aligned}
		\bm{\delta}^* & = \argmax_{\|\bm{\delta}\|\leq\rho} L(\bm{w} + \bm{\delta}) \approx \argmax_{\|\bm{\delta}\|\leq\rho} L(\bm{w}) + \bm{\delta}^T\nabla L(\bm{w}) \\
					 & = \rho\frac{\nabla L(\bm{w})}{\|\nabla L(\bm{w})\|}.
	\end{aligned}
\end{equation*}
Second, SAM updates $\bm{w}$ by adopting the approximate gradient of $L^{SAM}(\bm{w})$, which is 
\begin{equation*}
	\begin{aligned}
		&\nabla L^{SAM}(\bm{w}) \approx \nabla L(\bm{w} + \bm{\delta}^*) \\
       = & \nabla L(\bm{w}) |_{\bm{w} + \bm{\delta}^*} + \frac{d\bm{\delta}^*}{d\bm{w}} \nabla L(\bm{w}) |_{\bm{w} + \bm{\delta}^*} \approx \nabla L(\bm{w}) |_{\bm{w} + \bm{\delta}^*},
	\end{aligned}
\end{equation*}
where the second approximation is for accelerating the computation. Other gradient based optimizers (called base optimizer) can be incorporated into a generic framework of SAM, defined in Algorithm~\ref{alg:sam}. By varying $\bm{m}_t$ and $B_t$ of Algorithm~\ref{alg:sam}, we can obtain different base optimizers for SAM, such as \textsc{Sgd} \cite{sgd}, \textsc{SgdM} \cite{momentum2} and \textsc{Adam} \cite{adam}, see Tab.~\ref{tab:base_opt}. Note that when the base optimizer is SGD, Algorithm~\ref{alg:sam} rolls back to the original SAM in \citet{sam}.

\begin{algorithm}[htbp]
	\caption{Generic framework of SAM}
	\label{alg:sam}
	\begin{algorithmic}[1]
		\STATE {\bfseries Input:} parameters $\rho, \epsilon > 0$,
		$\bm{w}_1 \in \mathbb{R}^n$, step size $\{\alpha_t\}_{t=1}^T$, sequence of functions $\{\phi_t, \psi_t\}_{t=1}^T$
		\FOR{$t=1$ {\bfseries to} $T$}
		\STATE $\tilde{\bm{g}}_t = \nabla \ell_t(\bm{w}_t)$
		\STATE $\bm{\delta}_t = \rho\tilde{\bm{g}}_t / (\|\tilde{\bm{g}}_t\| + \epsilon)$
		\STATE $\bm{g}_t = \nabla \ell_t(\bm{w}_t + \bm{\delta}_t)$
		\STATE $\bm{m}_t = \phi_t(\bm{g}_1, \dots, \bm{g}_t)$ and $B_t = \psi_t(\bm{g}_1, \dots, \bm{g}_t)$
		\STATE $\bm{w}_{t+1} = \bm{w}_t - \alpha_t B_t^{-1} \bm{m}_t$
		\ENDFOR
	\end{algorithmic}
\end{algorithm}

\begin{table}[!hbtp]
	\caption{
		Base optimizers by different $\bm{m}_t$ and $B_t$.
	}
	\centering
	\label{tab:base_opt}
	\centering
	{\setlength{\tabcolsep}{3pt}
	\begin{tabular}{lcc}
		\toprule
		\textbf{Optimizer} & \bm{$m_t$} & \bm{$B_t$} \\
		\midrule
			\textsc{Sgd} & $\bm{g}_t$ & $\mathbb{I}$ \\
		
			\textsc{SgdM}  & $\sum_{i=0}^{t-1}\gamma^{i}\bm{g}_{t-i}$ & $\mathbb{I}$ \\
		
			\textsc{Adam} & $\frac{1 - \beta_{1}}{1 - \beta_{1}^t}\sum_{i=0}^{t-1}\bm{g}_{t-i}\beta_{1}^{i}$ & $\diag(\sqrt{\frac{1 - \beta_{2}}{1 - \beta_{2}^t}\sum_{i=1}^{t}\bm{g}_{i}^2\beta_{2}^{t-i}} + \epsilon)$ \\
	 	
		\bottomrule
	\end{tabular}
	}
\end{table}

	\section{Algorithm}
\label{sec:4}
\subsection{Details of WSAM}
\label{sec:4.1}
In this section, we give the formal definition of $L^{WSAM}$, which composes of a vanilla loss and a sharpness term. From Eq.~\eqref{eq:loss_sharp}, we have
\begin{equation}
	\begin{aligned}
		L^{WSAM}(\bm{w}) &:= L(\bm{w}) + \frac{\gamma}{1 - \gamma}\tilde{L}(\bm{w}) \\
		&= \frac{1 - 2\gamma}{1 - \gamma}L(\bm{w}) + \frac{\gamma}{1 - \gamma}L^{SAM}(\bm{w}),
	\end{aligned}
	\label{eq:WSAM}
\end{equation}
where $\gamma\in[0, 1)$. When $\gamma = 0$, $L^{WSAM}(\bm{w})$ degenerates to the vanilla loss; when $\gamma = 1 / 2$, $L^{WSAM}(\bm{w})$ is equivalent to $L^{SAM}(\bm{w})$; when $\gamma > 1 / 2$, $L^{WSAM}(\bm{w})$ gives more weights to the sharpness, and thus prone to find the point which has smaller curvature rather than smaller loss compared with SAM; and vice versa.

The generic framework of WSAM, which incorporates various optimizers by choosing different $\phi_t$ and $\psi_t$, is listed in Algorithm~\ref{alg:WSAM}. For example, when $\phi_t = \bm{g}_t$ and $\psi_t = \mathbb{I}$, we derive SGD with WSAM, which is listed in Algorithm~\ref{alg:SGD_WSAM}. Here, motivated by \citet{weight_decouple}, we adopt a \text{``weight decouple''} technique, i.e., the sharpness term $\tilde{L}(\bm{w})$ is not integrated into the base optimizer to calculate the gradients and update weights, but is calculated independently (the last term in Line 7 of Algorithm~\ref{alg:WSAM}). In this way the effect of the regularization just reflects the sharpness of the current step without additional information. For comparison, WSAM without \text{``weight decouple''}, dubbed Coupled-WSAM, is listed in Algorithm~\ref{alg:coupled-WSAM} of Section~\ref{sec:6.5}. For example, if the base optimizer is \textsc{SgdM} \cite{momentum2}, the update regularization term of Coupled-WSAM is the exponential moving averages of the sharpness. As shown in Section~\ref{sec:6.5}, \text{``weight decouple''} can improve the performance in most cases.

Fig.~\ref{fig:wsam-direction} depicts the WSAM update with different choices of $\gamma$. $\nabla{}L^{WSAM}(\bm{w})$ is between in $\nabla{}L(\bm{w})$ and $\nabla{}L^{SAM}(\bm{w})$ when $\gamma < \frac{1}{2}$, and gradually drift away $\nabla{}L(\bm{w})$ as $\gamma$ grows larger.

\begin{algorithm}[!htbp]
	\caption{Generic framework of WSAM}
	\label{alg:WSAM}
	\begin{algorithmic}[1]
		\STATE {\bfseries Input:} parameters $\rho,\epsilon > 0$, $\gamma\in[0,1)$,
		$\bm{w}_1 \in \mathbb{R}^n$, step size $\{\alpha_t\}_{t=1}^T$, sequence of functions $\{\phi_t, \psi_t\}_{t=1}^T$
		\FOR{$t=1$ {\bfseries to} $T$}
		\STATE $\tilde{\bm{g}}_t = \nabla \ell_t(\bm{w}_t)$
		\STATE $\bm{\delta}_t = \rho\tilde{\bm{g}}_t / (\|\tilde{\bm{g}}_t\| + \epsilon)$
		\STATE $\bm{g}_t = \nabla \ell_t(\bm{w}_t + \bm{\delta}_t)$
		\STATE $\tilde{\bm{m}}_t = \phi_t(\tilde{\bm{g}}_1, \dots, \tilde{\bm{g}}_t)$ and $\tilde{B}_t = \psi_t(\tilde{\bm{g}}_1, \dots, \tilde{\bm{g}}_t)$
		\STATE $\bm{w}_{t+1} = \bm{w}_t - \alpha_t \left(\tilde{B}_{t}^{-1}\tilde{\bm{m}}_t + \frac{\gamma}{1 - \gamma}(\bm{g}_t - \tilde{\bm{g}}_t)\right)$
		\ENDFOR
	\end{algorithmic}
\end{algorithm}

\begin{algorithm}[!htbp]
	\caption{SGD with WSAM}
	\label{alg:SGD_WSAM}
	\begin{algorithmic}[1]
		\STATE {\bfseries Input:} parameters $\rho,\epsilon > 0$, $\gamma\in[0,1)$,
		$\bm{w}_1 \in \mathbb{R}^n$, step size $\{\alpha_t\}_{t=1}^T$
		\FOR{$t=1$ {\bfseries to} $T$}
		\STATE $\tilde{\bm{g}}_t = \nabla \ell_t(\bm{w}_t)$
		\STATE $\bm{\delta}_t = \rho\tilde{\bm{g}}_t / (\|\tilde{\bm{g}}_t\| + \epsilon)$
		\STATE $\bm{g}_t = \nabla \ell_t(\bm{w}_t + \bm{\delta}_t)$
		\STATE $\bm{w}_{t+1} = \bm{w}_t - \alpha_t \left(\frac{\gamma}{1-\gamma}\bm{g}_t + \frac{1 - 2\gamma}{1 - \gamma}\tilde{\bm{g}}_t\right)$
		\ENDFOR
	\end{algorithmic}
\end{algorithm}

\begin{figure}[!htpb]
	\centering
	\includegraphics[width=0.8\linewidth]{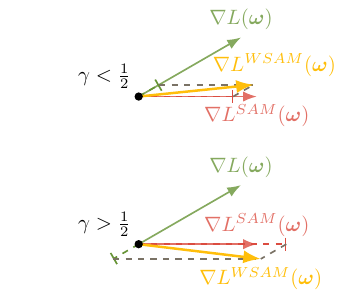}
	\caption{How WSAM updates on the choice of $\gamma$.}
	\label{fig:wsam-direction}
\end{figure}

\subsection{Toy Example}
To better illustrate the effect and benefit of $\gamma$ in WSAM, we setup a 2D toy example, similar to \citet{fsam}. As shown in Fig.~\ref{fig:wsam-traj}, the loss function contains a sharp minimum on the lower left (valued 0.28 at around $(-16.8, 12.8)$) and a flat minimum on the upper right (valued 0.36 at around $(19.8, 29.9)$).
The loss is defined as
\begin{equation*}
	L(\bm{w})
	= - \log \left(0.7e^{-K_1\left(\bm{w}\right)/1.8^2}+0.3e^{-K_2\left(\bm{w}\right)/1.2^2}\right),
\end{equation*}
while $K_i(\bm{w})=K_i(\mu, \sigma)$ is the KL divergence between the univariate Gaussian model and the two normal distributions, which is
\begin{equation*}
	K_i(\mu, \sigma) = \log \frac{\sigma_i}{\sigma} + \frac{\sigma^2+(\mu-\mu_i)^2}{2\sigma_i^2} - \frac{1}{2},
\end{equation*}
where $(\mu_1, \sigma_1) = (20, 30)$ and $(\mu_2, \sigma_2) = (-20, 10)$.

We use \textsc{SgdM} with momentum 0.9 as the base optimizer, and set $\rho=2$ for both SAM and WSAM. Starting from the initial point $(-6, 10)$, the loss function is optimized for 150 steps with a learning rate of 5.
SAM converges to the lower but sharper minimum, as well as WSAM with $\gamma=0.6$. However, a larger $\gamma=0.95$
leads to the flat minimum, because a stronger sharpness regularization comes to effect.

\begin{figure}[!htpb]
	\centering
	\includegraphics[width=0.8\linewidth]{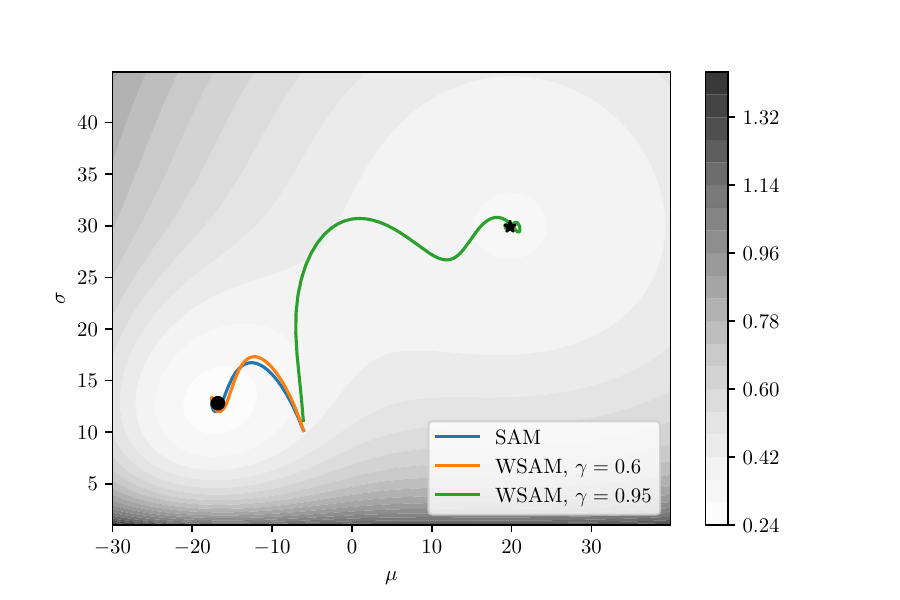}
	\caption{WSAM can achieve different minima by choosing different $\gamma$.}
	\label{fig:wsam-traj}
\end{figure}

\section{Theoretical Analysis}
\label{sec:5}
\subsection{Convergence of WSAM}
In this section, we choose SGD as the base optimizer for simplicity, i.e., Algorithm~\ref{alg:SGD_WSAM}, and use $\rho_t$ to replace $\rho$ of Algorithm~\ref{alg:SGD_WSAM} where $\rho_t$ is non-increasing.
We have the following convergence theorems for both convex and non-convex settings.
\begin{theorem}
	(Convergence in convex settings) Let $\{\bm{w}_t\}$ be the sequence  obtained by Algorithm \ref{alg:SGD_WSAM}, $\alpha_t = \alpha / \sqrt{t}$, $\rho_{t}\leq\rho$, $\|\bm{g}_t\|_{\infty} \leq G_{\infty}, \|\tilde{\bm{g}}_t\|_{\infty} \leq G_{\infty} \ \forall{}t\in[T]$. Suppose $\ell_t(\bm{w})$ is convex and $\eta$-smooth, i.e., $\|\nabla\ell_t(\bm{u}) - \nabla\ell_t(\bm{v})\| \leq \eta\|\bm{u} - \bm{v}\|$, for all $t\in{}[T]$, $\bm{w}^*$ is an optimal solution of $\sum_{t=1}^{T}\ell_t(\bm{w})$, i.e., $\bm{w}^* = \argmin_{\bm{w}\in\mathbb{R}^n} \sum_{t=1}^{T}\ell_t(\bm{w})$ and there exists the constant $D_{\infty}$ such that $\max_{t\in[T]}\|\bm{w}_t - \bm{w}^*\|_{\infty}\leq D_{\infty}$. Then we have the following bound on the regret
	\begin{equation*}
		\begin{aligned}
			&\sum_{t=1}^{T}\left(\ell_t(\bm{w}_t) - \ell_t(\bm{w}^*)\right) \leq C_1\sqrt{T} + C_2\sum_{t=1}^{T}\rho_{t},
		\end{aligned}
	\end{equation*}
	\label{th:convex}
	where $C_1$ and $C_2$ are defined as follows:
	\begin{equation*}
		C_1 = \frac{nD_{\infty}^2}{2\alpha} + \frac{10\gamma^2 - 8\gamma + 2}{(1 - \gamma)^2}nG_{\infty}^2, \ C_2 = \frac{\sqrt{n}D_{\infty}\eta\gamma}{1 - \gamma}.
	\end{equation*}
\end{theorem}

Here, to ensure the condition $\max_{t\in[T]}\|\bm{w}_t - \bm{w}^*\|_{\infty}\leq D_{\infty}$ holds, we can make the assumption that the domain $\mathcal{W}\subseteq\mathbb{R}^n$ is bounded and project the sequence $\{\bm{w}_t\}$ onto $\mathcal{W}$ by setting $\bm{w_{t+1}} = \Pi{}_{\mathcal{W}}\left(\bm{w}_t - \alpha_t \left(\frac{\gamma}{1-\gamma}\bm{g}_t + \frac{1 - 2\gamma}{1 - \gamma}\tilde{\bm{g}}_t\right)\right)$.

\begin{proof}
	Let $\bm{h}_t = \frac{\gamma}{1 - \gamma}\bm{g}_t + \frac{1 - 2\gamma}{1 - \gamma}\tilde{\bm{g}}_t$. Since
	\begin{equation}
		\begin{aligned}
			&\|\bm{w}_{t+1} - \bm{w}^*\|^2 = \left\|\bm{w}_t - \alpha_t\left(\frac{\gamma}{1 - \gamma}\bm{g}_t + \frac{1 - 2\gamma}{1 - \gamma}\tilde{\bm{g}}_t\right) - \bm{w}^*\right\|^2 = \|\bm{w}_t - \bm{w}^*\|^2 \\
			& - 2\alpha_t\left<\bm{w}_t - \bm{w}^*, \tilde{\bm{g}}_t\right> - 2\alpha_{t}\left<\bm{w}_t - \bm{w}^*, \frac{\gamma}{1 - \gamma}(\bm{g}_t - \tilde{\bm{g}}_t)\right> + \alpha_t^2\left\| \bm{h}_t \right\|^2,
		\end{aligned}
		\label{eq:th1_1}
	\end{equation}
	then rearranging Eq.~\eqref{eq:th1_1}, we have
	\begin{equation*}
		\begin{aligned}
			&\left<\bm{w}_t - \bm{w}^*, \tilde{\bm{g}}_t\right> = \frac{1}{2\alpha_{t}}\left(\|\bm{w}_t - \bm{w}^*\|^2 - \|\bm{w}_{t+1} - \bm{w}^*\|^2\right) \\
			& - \frac{\gamma}{1-\gamma}<\bm{w_t} - \bm{w}^*, \bm{g}_t - \tilde{\bm{g}}_t> + \frac{\alpha_{t}}{2}\|\bm{h}_t\|^2  \\
			\leq& \frac{1}{2\alpha_{t}}\left(\|\bm{w}_t - \bm{w}^*\|^2 - \|\bm{w}_{t+1} - \bm{w}^*\|^2\right) + \frac{\gamma}{1-\gamma}\|\bm{w}_t - \bm{w}^*\|\|\bm{g}_t - \tilde{\bm{g}}_t\| \\
			& + \alpha_{t}\left(\frac{(1-2\gamma)^2}{(1-\gamma)^2}\|\bm{g}_t\|^2 + \frac{\gamma^2}{(1-\gamma)^2}\|\tilde{\bm{g}}_t\|^2\right)  \\
			\leq& \frac{1}{2\alpha_{t}}\left(\|\bm{w}_t - \bm{w}^*\|^2 - \|\bm{w}_{t+1} - \bm{w}^*\|^2\right) + \frac{\gamma}{1-\gamma}\sqrt{n}D_{\infty}\eta\rho_{t} \\
			& + \frac{5\gamma^2 - 4\gamma + 1}{(1 - \gamma)^2}nG_{\infty}^2\alpha_{t},
		\end{aligned}
	\end{equation*}
	where the first inequality follows from Cauchy-Schwartz inequality and $ab\leq\frac{1}{2}(a^2 + b^2)$, and the second inequality follows from the $\eta$-smooth of $\ell_t(\bm{w})$, $\|\bm{w}_t - \bm{w}^*\|_{\infty}\leq{}D_{\infty}$ and $\|\bm{g}_t\|_{\infty}\leq{}G_{\infty}$. Thus, the regret
	\begin{equation*}
		\begin{aligned}
			& \sum_{t=1}^{T}\left(\ell_t(\bm{w}_t) - \ell_t(\bm{w}^*)\right) \leq\sum_{t=1}^{T}\left<\bm{w}_t - \bm{w}^*,
			\tilde{\bm{g}}_t\right> \\
			\leq& \sum_{t=1}^{T}\bigg[\frac{1}{2\alpha_{t}}\left(\|\bm{w}_t - \bm{w}^*\|^2 - \|\bm{w}_{t+1} - \bm{w}^*\|^2\right) + \frac{\gamma}{1-\gamma}\sqrt{n}D_{\infty}\eta\rho_{t} \\
			& + \frac{5\gamma^2 - 4\gamma + 1}{(1 - \gamma)^2}nG_{\infty}^2\alpha_{t}\bigg] \\
			\leq& \frac{1}{2\alpha_{1}}\|\bm{w}_1 - \bm{w}^*\|^2 + \sum_{t=2}^{T}\frac{\|\bm{w}_t - \bm{w}^*\|^2}{2}\left(\frac{1}{\alpha_{t}} - \frac{1}{\alpha_{t-1}}\right) + \sum_{t=1}^{T}\frac{\sqrt{n}D_{\infty}\eta\gamma}{1 - \gamma}\rho_{t} \\
			& + \sum_{t=1}^{T}\frac{5\gamma^2 - 4\gamma + 1}{(1 - \gamma)^2}nG_{\infty}^2\alpha_{t} \\
			\leq& \frac{nD_{\infty}^2}{2\alpha} + \frac{nD_{\infty}^2}{2}\sum_{t=2}^{T}\left(\frac{1}{\alpha_{t}} - \frac{1}{\alpha_{t-1}}\right) + \sum_{t=1}^{T}\frac{\sqrt{n}D_{\infty}\eta\gamma}{1 - \gamma}\rho_{t} \\
			& + \sum_{t=1}^{T}\frac{5\gamma^2 - 4\gamma + 1}{(1 - \gamma)^2}nG_{\infty}^2\alpha_{t} \\
			\leq& \frac{nD_{\infty}^2}{2\alpha}\sqrt{T} + \sum_{t=1}^{T}\frac{\sqrt{n}D_{\infty}\eta\gamma}{1 - \gamma}\rho_{t} + \frac{10\gamma^2 - 8\gamma + 2}{(1 - \gamma)^2}nG_{\infty}^2\sqrt{T},
		\end{aligned}
	\end{equation*}
	where the first inequality follows from the convexity of $\ell_t(\bm{w})$ and the last inequality follows from
	\begin{equation*}
		\begin{aligned}
			\sum_{t=1}^{T}\frac{1}{\sqrt{t}} &= 1 + \int_{2}^{3}\frac{1}{\sqrt{2}}ds + \cdots + \int_{T-1}^{T}\frac{1}{\sqrt{T}}ds \\
			& < 1 + \int_{2}^{3}\frac{1}{\sqrt{s - 1}}ds + \cdots + \int_{T-1}^{T}\frac{1}{\sqrt{s - 1}}ds \\
			& = 1 + \int_{2}^{T}\frac{1}{\sqrt{s - 1}}ds = 2\sqrt{T-1} - 1 < 2\sqrt{T}.
		\end{aligned}
	\end{equation*}
	This completes the proof.
\end{proof}

\begin{corollary}
	Suppose $\rho_{t} = \rho / \sqrt{t}$, then we have
	\begin{equation*}
		\begin{aligned}
			&\sum_{t=1}^{T}\left(\ell_t(\bm{w}_t) - \ell_t(\bm{w}^*)\right) \leq (C_1 + 2\rho{}C_2)\sqrt{T},
		\end{aligned}
	\end{equation*}
	where $C_1$ and $C_2$ are the same with Theorem~\ref{th:convex}.
	\label{coro:convex}
\end{corollary}

\begin{proof}
	Since $\rho_{t} = \rho / \sqrt{t}$, we have
	\begin{equation*}
		\begin{aligned}
			\sum_{t=1}^{T}\rho_{t} &= \rho\sum_{t=1}^{T}\frac{1}{\sqrt{t}} < 2\rho\sqrt{T}.
		\end{aligned}
	\end{equation*}
	This completes the proof.
\end{proof}

 Corollary~\ref{coro:convex} implies the regret is $O(\sqrt{T})$ and can achieve the convergence rate $O(1 / \sqrt{T})$ in convex settings.

\begin{theorem}
	(Convergence in non-convex settings) Suppose that the following assumptions are satisfied:
	\begin{enumerate}[leftmargin=14pt]
		\item $L$ is differential and lower bounded, i.e., $L(\bm{w}^*) > -\infty$ where $\bm{w}^*$ is an optimal solution. $L$ is also $\eta$-smooth, i.e., $\forall{}\bm{u}, \bm{v}\in\mathbb{R}^n$, we have
		$$L(\bm{u}) \leq L(\bm{v}) + \left<\nabla{}L(\bm{v}), \bm{u} - \bm{v}\right> + \frac{\eta}{2}\|\bm{u} - \bm{v}\|^2.$$
		\item At step $t$, the algorithm can access the bounded noisy gradients and the true gradient is bounded, i.e., $\|\bm{g}_t\|_{\infty}\leq G_{\infty}, \|\tilde{\bm{g}}_t\|_{\infty} \leq G_{\infty}, \|\nabla{}L(\bm{w}_t)\|_{\infty}\leq G_{\infty}, \forall{}t\in[T]$.
		\item The noisy gradient is unbiased and the noise is independent, i.e., $\tilde{\bm{g}}_t = \nabla{}L(\bm{w}_t) + \bm{\zeta}_{t}, \mathbb{E}[\bm{\zeta}_{t}] = \bm{0}$ and $\bm{\zeta}_{i}$ is independent of $\bm{\zeta}_{j}$ if $i\neq j$.
		\item $\alpha_t = \alpha / \sqrt{t}$, $\rho_{t}\leq\rho$,  $\forall{}t\in[T]$.
	\end{enumerate}
	Then Algorithm~\ref{alg:SGD_WSAM} yields
	\begin{equation*}
		\begin{aligned}
			& \min_{t\in[T]}\mathbb{E}[\|\nabla{}L(\bm{w}_t)\|^2] \\
	  \leq& \frac{1}{\sqrt{T} - 1}\left(C_3 + C_5 + C_4\sum_{t=1}^{T}\alpha_{t}\rho_{t} + C_5\log{}T\right),
		\end{aligned}
	\end{equation*}
	where $C_3$, $C_4$ and $C_5$ are defined as follows:
	\begin{equation*}
		\begin{aligned}
			C_3 &= \frac{L(\bm{w}_1)}{2\alpha}, \ C_4 = \frac{\sqrt{n}G_{\infty}\eta\gamma}{2\alpha(1 - \gamma)}, \\
			C_5 &= \frac{5\gamma^2 - 4\gamma + 1}{2(1 - \gamma)^2}nG_{\infty}^2\eta\alpha.
		\end{aligned}
	\end{equation*}
	\label{th:non-convex}
\end{theorem}

\begin{proof}
	By assumption 1, we have
	\begin{equation}
		\begin{aligned}
			& L(\bm{w}_{t+1}) \leq L(\bm{w}_t) + \left<\nabla{}L(\bm{w}_t), \bm{w}_{t+1} - \bm{w}_t\right> + \frac{\eta}{2}\|\bm{w}_{t+1} - \bm{w}_t\|^2 \\
			=& L(\bm{w}_t) - \left<\nabla{}L(\bm{w}_t), \alpha_{t}\left(\tilde{\bm{g}}_t + \frac{\gamma}{1 - \gamma}\left(\bm{g}_t - \tilde{\bm{g}}_t\right)\right)\right> +  \frac{\eta}{2}\|\bm{w}_{t+1} - \bm{w}_t\|^2.
		\end{aligned}
		\label{eq:th2_3}
	\end{equation}
	Rearranging Eq.~\eqref{eq:th2_3} and taking expectation both sides, by assumptions 3, we get
	\begin{equation}
		\begin{aligned}
			& \alpha_{t}\mathbb{E}[\|\nabla{}L(\bm{w}_t)\|^2] \leq \mathbb{E}[L(\bm{w}_t) - L(\bm{w}_{t + 1})] \\
			& - \alpha_{t}\mathbb{E}\left[\left<\nabla{}L(\bm{w}_t), \frac{\gamma}{1 - \gamma}\left(\nabla{}L(\bm{w}_t + \bm{\delta}_t) - \nabla{}L(\bm{w}_t)\right)\right>\right] + \frac{\eta}{2}\|\bm{w}_{t+1} - \bm{w}_t\|^2 \\
			\leq& \mathbb{E}[L(\bm{w}_t) - L(\bm{w}_{t + 1})] + \alpha_{t}\frac{\gamma}{1 - \gamma}\mathbb{E}[\|\nabla{}L(\bm{w}_t)\|\|\nabla{}L(\bm{w}_t + \bm{\delta}_t) - \nabla{}L(\bm{w}_t)\|] \\
			& + \frac{\eta}{2}\alpha_{t}^2\left\|\frac{\gamma}{1 - \gamma}\bm{g}_t + \frac{1- 2\gamma}{1-\gamma}\tilde{\bm{g}}_t\right\|^2 \\
			\leq& \mathbb{E}[L(\bm{w}_t) - L(\bm{w}_{t + 1})] + \frac{\sqrt{n}G_{\infty}\eta\gamma}{1 - \gamma}\alpha_{t}\rho_{t} + \frac{5\gamma^2 - 4\gamma + 1}{(1 - \gamma)^2}nG_{\infty}^2\eta\alpha_{t}^2.
		\end{aligned}
		\label{eq:th2_4}
	\end{equation}
	Telescoping Eq.~\eqref{eq:th2_4} for $t = 1$ to $T$, we have
	\begin{equation}
		\begin{aligned}
			& \sum_{t=1}^{T}\alpha_{t}\mathbb{E}[\|\nabla{}L(\bm{w}_t)\|^2] \leq \mathbb{E}[L(\bm{w}_1) - L(\bm{w}_{T+1})] + \frac{\sqrt{n}G_{\infty}\eta\gamma}{1 - \gamma}\sum_{t=1}^{T}\alpha_{t}\rho_{t} \\
			& +  \frac{5\gamma^2 - 4\gamma + 1}{(1 - \gamma)^2}nG_{\infty}^2\eta\sum_{t=1}^{T}\alpha_{t}^2 \\
			\leq& L(\bm{w}_1) + \frac{\sqrt{n}G_{\infty}\eta\gamma}{1 - \gamma}\sum_{t=1}^{T}\alpha_{t}\rho_{t} +  \frac{5\gamma^2 - 4\gamma + 1}{(1 - \gamma)^2}nG_{\infty}^2\eta\sum_{t=1}^{T}\alpha_{t}^2.
		\end{aligned}
		\label{eq:th2_5}
	\end{equation}
	Since
	\begin{equation}
		\begin{aligned}
			\sum_{t = 1}^{T}\alpha_t &=  \sum_{t = 1}^{T}\frac{\alpha}{\sqrt{t}} = \alpha\left(\int_{1}^{2}\frac{1}{\sqrt{1}}ds + \cdots + \int_{T - 1}^{T}\frac{1}{\sqrt{T}}ds\right) \\
			&> \alpha\int_{1}^{T}\frac{1}{\sqrt{s}}ds = 2\alpha\left(\sqrt{T} - 1\right), \\
			\sum_{t = 1}^{T}\alpha_t^2 &=  \sum_{t = 1}^{T}\frac{\alpha^2}{t} = \alpha^2\left(1 + \int_{2}^{3}\frac{1}{2}ds + \cdots + \int_{T-1}^{T}\frac{1}{T}ds\right) \\
			&< \alpha^2\left(1 + \int_{2}^{T}\frac{1}{s - 1}ds\right) = \alpha^2\left(\log(T - 1) + 1 \right) \\
			&< \alpha^2\left(\log{}T + 1 \right),
		\end{aligned}
		\label{eq:th2_6}
	\end{equation}
	substituting Eq.~\eqref{eq:th2_6} into Eq.~\eqref{eq:th2_5}, we have
	\begin{equation}
		\begin{aligned}
			& \min_{t\in[T]}\mathbb{E}[\|\nabla{}L(\bm{w}_t)\|^2] \\
			\leq& \frac{1}{\sum_{t=1}^{T}\alpha_{t}}\bigg( L(\bm{w}_1) + \frac{\sqrt{n}G_{\infty}\eta\gamma}{1 - \gamma}\sum_{t=1}^{T}\alpha_{t}\rho_{t} +  \frac{5\gamma^2 - 4\gamma + 1}{(1 - \gamma)^2}nG_{\infty}^2\eta\sum_{t=1}^{T}\alpha_{t}^2 \bigg) \\
			\leq& \frac{1}{\sqrt{T} - 1} \bigg(\frac{L(\bm{w}_1)}{2\alpha} + \frac{\sqrt{n}G_{\infty}\eta\gamma}{2\alpha(1 - \gamma)}\sum_{t=1}^{T}\alpha_{t}\rho_{t} \\
			&+ \frac{5\gamma^2 - 4\gamma + 1}{2(1 - \gamma)^2}nG_{\infty}^2\eta\alpha(\log{}T + 1)\bigg).
		\end{aligned}
		\label{eq:th2_final}
	\end{equation}
	This completes the proof.
\end{proof}

Then we have the following corollary.

\begin{corollary}
	Suppose $\rho_{t} = \rho / \sqrt{t}$, then we have
	\begin{equation*}
		\begin{aligned}
			&\min_{t\in[T]}\mathbb{E}[\|\nabla{}f(\bm{w}_t)\|^2] \\
	  \leq& \frac{1}{\sqrt{T} - 1}\bigg(C_3 + C_5 + \alpha\rho{}C_4 + (C_5 + \alpha\rho{}C_4)\log{T}\bigg),
		\end{aligned}
	\end{equation*}
	where $C_3\sim{}C_5$ are the same with Theorem~\ref{th:non-convex}.
	\label{coro:non-convex}
\end{corollary}

\begin{proof}
	Since $\rho_{t} = \rho / \sqrt{t}$, we have
	\begin{equation}
		\begin{aligned}
			\sum_{t=1}^{T}\alpha_{t}\rho_{t} = \alpha\rho\sum_{t=1}^{T}\frac{1}{t} < \alpha\rho(\log{T} + 1).
		\end{aligned}
		\label{eq:coro2_1}
	\end{equation}
	Substituting Eq.~\eqref{eq:coro2_1} into Eq.~\eqref{eq:th2_final}, we finish the proof.
\end{proof}

Corollaries~\ref{coro:non-convex} implies the convergence (to the stationary point) rate of WSAM is $O(\log{T} / \sqrt{T})$ in non-convex settings.

\subsection{Generalization of WSAM}
In this section, we are interested in binary classification problems, i.e., the label set $\mathcal{Y} = \{0, 1\}$, and focus on the 0-1 loss, i.e., $\ell(h_{\bm{w}}(\bm{x}), y) = \mathbf{I}(h_{\bm{w}}(\bm{x}) \neq y)$ where $\mathbf{I}$ is the indicator function. Followed by \citet{uniform_convergence, uml, pac-bayes, nonvacuous_gen, sam}, we have the following generalization property.
\begin{theorem}
	Let $\mathcal{H} = \{h_{\bm{w}}: \bm{w}\in\mathbb{R}^n\}$ be a hypothesis class of functions from a domain $\mathcal{X}$ to $\{0, 1\}$ and let the loss function be the 0-1 loss. Assume that $\text{VCdim}(\mathcal{H}) = d < \infty$ and $L_{\mathcal{D}}(\bm{w}) \leq\mathbb{E}_{\bm{\epsilon}\sim{}N(\bm{0}, \rho^2\mathbb{I})}[L_{\mathcal{D}}(\bm{w} + \bm{\epsilon})]$. Then for any $\rho > 0$, $\gamma\in[0, 1)$ and any distribution $\mathcal{D}$, with probability of at least $1 - \delta$ over the choice of the training set $\mathcal{S}$ which has $m$ elements drawn i.i.d. according to $\mathcal{D}$, we have
	\begin{equation*}
		L_{\mathcal{D}}(\bm{w}) \leq L_{\mathcal{S}}^{WSAM}(\bm{w}) + \frac{2|1 - 2\gamma|}{1 - \gamma}\sqrt{\frac{C_1}{m}} + \frac{\gamma}{1 - \gamma}\sqrt{\frac{C_2 + C_3}{m - 1}}.
	\end{equation*}
	where $L_{\mathcal{S}}^{WSAM}$ is defined in Eq.~\eqref{eq:WSAM} and $C_1\sim{}C_3$ are defined as follows:
	\begin{equation*}
		\begin{aligned}
			C_1 &= 8d\log(em / d) + 2\log(4 / \delta), \\
			C_2 &= n\log\left(1 + \frac{\|\bm{w}\|^2}{\rho^2}\left(1 + \sqrt{\frac{\log(m)}{n}}\right)^2\right), \\
			C_3 &= 4\log(m / \delta) + 8\log(6m + 3n).
		\end{aligned}
	\end{equation*}
	\label{th:generalization}
\end{theorem}

\begin{proof}
	From Section 28.1 of \citet{uml} and Theorem 2 of \citet{sam}, we have
	\begin{equation*}
		\begin{aligned}
			& |L_{\mathcal{D}}(\bm{w}) - L_{\mathcal{S}}(\bm{w})| \leq 2\sqrt{\frac{8d\log(em / d) + 2\log(4 / \delta)}{m}}, \\
			& L_{\mathcal{D}}(\bm{w}) \leq \max_{\|\bm{\epsilon}\|\leq\rho}L_{\mathcal{S}}(\bm{w} + \bm{\epsilon}) + \frac{1}{\sqrt{m-1}}\Bigg(4\log(m / \delta) + 8\log(6m + 3n)  \\
			& \qquad\qquad + n\log\bigg(1 + \frac{\|\bm{w}\|^2}{\rho^2}\big(1 + \sqrt{\frac{\log(m)}{n}}\big)^2\bigg) \Bigg)^{1/2}.
		\end{aligned}
	\end{equation*}
	Hence, we have
	\begin{equation*}
		\begin{aligned}
			L_{\mathcal{D}}(\bm{w}) = & \frac{1 - 2\gamma}{1 - \gamma}L_{\mathcal{D}}(\bm{w}) + \frac{\gamma}{1 - \gamma}L_{\mathcal{D}}(\bm{w}) \\
			\leq& \frac{1 - 2\gamma}{1 - \gamma}L_{\mathcal{S}}(\bm{w}) + \frac{2|1 - 2\gamma|}{1 - \gamma}\sqrt{\frac{8d\log(em / d) + 2\log(4 / \delta)}{m}} \\
			& + \frac{\gamma}{1 - \gamma}\max_{\|\bm{\epsilon}\|\leq\rho}L_{\mathcal{S}}(\bm{w} + \bm{\epsilon}) + \frac{\gamma}{1 - \gamma}\frac{1}{\sqrt{m-1}}\Bigg(4\log(m / \delta) \\
			& + 8\log(6m + 3n) + n\log\bigg(1 + \frac{\|\bm{w}\|^2}{\rho^2}\big(1 + \sqrt{\frac{\log(m)}{n}}\big)^2\bigg) \Bigg)^{1/2} \\
			=& L_{\mathcal{S}}^{WSAM}(\bm{w}) + \frac{2|1 - 2\gamma|}{1 - \gamma}\sqrt{\frac{8d\log(em / d) + 2\log(4 / \delta)}{m}} \\
			& + \frac{\gamma}{1 - \gamma}\frac{1}{\sqrt{m-1}}\Bigg(n\log\bigg(1 + \frac{\|\bm{w}\|^2}{\rho^2}\big(1 + \sqrt{\frac{\log(m)}{n}}\big)^2\bigg) \\
			& + 4\log(m / \delta) + 8\log(6m + 3n) \Bigg)^{1/2}.
		\end{aligned}
	\end{equation*}
	This completes the proof.
\end{proof}

Note that we assume $\rho$ ($\rho_{t}$) decreases to zero for proving the convergence in both convex and non-convex settings. However, the generalization bound would go to infinity if $\rho$ decreases to zero. In practice, we keep $\rho$ be a constant. To prove the convergence when $\rho$ is constant would be an interesting problem for the future work.

	\section{Experiments}
\label{sec:6}
In this section, we conduct experiments with a wide range of tasks to verify the effectiveness of WSAM.

\subsection{Image Classification from Scratch}
\label{subsec:image-scratch}
We first study WSAM's performance for training models from scratch on Cifar10 and Cifar100 datasets.
The models we choose include ResNet18 \cite{He2015} and WideResNet-28-10 \cite{wrn}.
We train the models on both Cifar10 and Cifar100 with a predefined batch size, 128 for ResNet18 and 256 for WideResNet-28-10.
The base optimizer used here is \textsc{SgdM} with momentum 0.9.
Following the settings of SAM \cite{sam}, each vanilla training runs twice as many epochs as a SAM-like training run.
We train both models for 400 epochs (200 for SAM-like optimizers), and use cosine scheduler to decay the learning rate.
Note that we do not use any advanced data augmentation methods, such as cutout regularization \cite{Cutout} and AutoAugment \cite{autoaugment}.

For both models, we determine the learning rate and weight decay using a joint grid search for vanilla training,
and keep them invariant for the next SAM experiments.
The search range is \{0.05, 0.1\} and \{1e-4, 5e-4, 1e-3\} for learning rate and weight decay, respectively.
Since all SAM-like optimizers have a hyperparameter $\rho$ (the neighborhood size), we then search for the best $\rho$ over the SAM optimizer,
and use the same value for other SAM-like optimizers.
The search range for $\rho$ is \{0.01, 0.02, 0.05, 0.1, 0.2, 0.5\}.
We conduct independent searches over each optimizer individually for optimizer-specific hyperparameters and report the best performance.
We use the range recommended in the corresponding article for the search.
For GSAM, we search for $\alpha$ in \{0.01, 0.02, 0.03, 0.1, 0.2, 0.3\}.
For ESAM, we search for $\beta$ in \{0.4, 0.5, 0.6\} and $\gamma$ in \{0.4, 0.5, 0.6\}.
For WSAM, we search for $\gamma$ in \{0.5, 0.6, 0.7, 0.8, 0.82, 0.84, 0.86, 0.88, 0.9, 0.92, 0.94, 0.96\}.
We repeat the experiments five times with different random seeds and report the mean error and the associated standard deviation.
We conduct the experiments on a single NVIDIA A100 GPU.
Hyperparameters of the optimizers for each model are summarized in Tab.~\ref{tab:hyper-cifar}.

Tab.~\ref{tab:cifar-results-acc} gives the top-1 error for ResNet18, WRN-28-10 trained on Cifar10 and Cifar100 with different optimizers.
SAM-like optimizers improve significantly over the vanilla one, and WSAM outperforms the other SAM-like optimizers for both models on Cifar10/100.

\begin{table}[!htbp]
	\caption{Top-1 error (\%) for ResNet18, WRN-28-10 trained on Cifar10 and Cifar100 with different optimizers.}
	\medskip
	\label{tab:cifar-results-acc}
	\centering
	\begin{tabular}{cccc}
		\toprule
		                         &               & Cifar10                    & Cifar100                    \\
		\midrule
		\multirow{4}*{ResNet18}  & Vanilla (\textsc{SgdM}) & $ 4.32 \pm 0.07 $          & $ 20.51 \pm 0.20 $          \\
		                         & SAM           & $ 3.68 \pm 0.06 $          & $ 19.97 \pm 0.30 $          \\
		                         & ESAM          & $ 3.83 \pm 0.08 $          & $ 20.67 \pm 0.24 $          \\
		                         & GSAM          & $ 3.71 \pm 0.11 $          & $ 19.84 \pm 0.27 $          \\
		                         & WSAM          & $ \mathbf{3.62 \pm 0.10} $ & $ \mathbf{19.42 \pm 0.12} $ \\
		\midrule
		\multirow{4}*{WRN-28-10} & Vanilla (\textsc{SgdM}) & $ 3.73 \pm 0.09 $          & $ 19.27 \pm 0.12 $          \\
		                         & SAM           & $ 2.94 \pm 0.08 $          & $ 16.49 \pm 0.09 $          \\
		                         & ESAM          & $ 2.99 \pm 0.07 $          & $ 16.52 \pm 0.21 $          \\
		                         & GSAM          & $ 2.91 \pm 0.08 $          & $ 16.36 \pm 0.32 $          \\
		                         & WSAM          & $ \mathbf{2.74 \pm 0.07} $ & $ \mathbf{16.33 \pm 0.26} $ \\
		\bottomrule
	\end{tabular}
\end{table}

\begin{table*}[!htbp]
	\caption{Hyperparameters to reproduce experimental results on Cifar10 and Cifar100.}
	\medskip
	\label{tab:hyper-cifar}
	\centering
	\begin{tabular}{ccccccccccc}
		\toprule
		              & \multicolumn{5}{c}{Cifar10} & \multicolumn{5}{c}{Cifar100}                                                       \\
		\midrule
		ResNet18      & \textsc{SgdM}                 & SAM                           & ESAM & GSAM & WSAM & \textsc{SgdM} & SAM & ESAM & GSAM & WSAM \\
		\midrule
		learning rate & \multicolumn{5}{c}{0.05}     & \multicolumn{5}{c}{0.05}                                                           \\
		weight decay  & \multicolumn{5}{c}{1e-3}     & \multicolumn{5}{c}{1e-3}                                                           \\
		$\rho$        & -                             & 0.2                           & 0.2  & 0.2  & 0.2  & -   & 0.2 & 0.2  & 0.2  & 0.2  \\
		$\alpha$      & -                             & -                             & -    & 0.02 & -    & -   & -   & -    & 0.03 & -    \\
		$\beta$       & -                             & -                             & 0.6  & -    & -    & -   & -   & 0.5  & -    & -    \\
		$\gamma$      & -                             & -                             & 0.6  & -    & 0.88 & -   & -   & 0.5  & -    & 0.82 \\
		\midrule
		WRN-28-10     & \textsc{SgdM}               & SAM                           & ESAM & GSAM & WSAM & \textsc{SgdM} & SAM & ESAM & GSAM & WSAM \\
		\midrule
		learning rate & \multicolumn{5}{c}{0.1}      & \multicolumn{5}{c}{0.1}                                                            \\
		weight decay  & \multicolumn{5}{c}{1e-3}     & \multicolumn{5}{c}{1e-3}                                                           \\
		$\rho$        & -                             & 0.2                           & 0.2  & 0.2  & 0.2  & -   & 0.2 & 0.2  & 0.2  & 0.2  \\
		$\alpha$      & -                             & -                             & -    & 0.01 & -    & -   & -   & -    & 0.2  & -    \\
		$\beta$       & -                             & -                             & 0.6  & -    & -    & -   & -   & 0.5  & -    & -    \\
		$\gamma$      & -                             & -                             & 0.6  & -    & 0.88 & -   & -   & 0.6  & -    & 0.94 \\
		\bottomrule
	\end{tabular}
\end{table*}

\subsection{Extra Training on ImageNet}
We further experiment on the ImageNet dataset \cite{ILSVRC15} using Data-Efficient Image Transformers \cite{vit_model}. We restore a pre-trained DeiT-base checkpoint\footnote{https://github.com/facebookresearch/deit}, and continue training for three epochs. The model is trained using a batch size of 256, \textsc{SgdM} with momentum 0.9 as the base optimizer, a weight decay of 1e-4, and a learning rate of 1e-5. We conduct the experiment on four NVIDIA A100 GPUs.

We search for the best $\rho$ for SAM in $\{0.05, 0.1, 0.5, 1.0, \cdots, 6.0\}$. The best $\rho=5.5$ is directly used without further tunning by GSAM and WSAM. After that, we search for the optimal $\alpha$ for GSAM in $\{0.01, 0.02, 0.03, 0.1, 0.2, 0.3\}$ and $\gamma$ for WSAM from 0.80 to 0.98 with a step size of 0.02.

The initial top-1 error rate of the model is $18.2\%$, and the error rates after the three extra epochs are shown in Tab.~\ref{tab:over-training-res}. We find no significant differences between the three SAM-like optimizers, while they all outperform the vanilla optimizer, indicating that they find a flatter minimum with better generalization.

\begin{table}[!htbp]
	\caption{Top-1 error (\%) for Deit-Base over-trained on ImageNet.}
	\medskip
	\label{tab:over-training-res}
	\centering
	\begin{tabular}{cc}
		\toprule
		                     & Top-1 error (\%)  \\
		\midrule
		Initial              & $18.2$            \\
		Vanilla (\textsc{SgdM})        & $18.17 \pm 0.005$ \\
		SAM                  & $18.01 \pm 0.007$ \\
		GSAM ($\alpha=0.02$) & $18.01 \pm 0.005$ \\
		WSAM ($\gamma=0.94$) & $18.01 \pm 0.003$ \\
		\bottomrule
	\end{tabular}
\end{table}

\subsection{Robustness to Label Noise}
As shown in previous works \cite{sam, asam, fsam}, SAM-like optimizers exhibit good robustness to label noise in the training set,
on par with those algorithms specially designed for learning with noisy labels \cite{Arazo2019, Jiang2020}.
Here, we compare the robustness of WSAM to label noise with SAM, ESAM, and GSAM.
We train a ResNet18 for 200 epochs on the Cifar10 dataset and inject symmetric label noise of noise levels $ 20\%$, $ 40\%$, $ 60\%$, and $ 80\%$ to the training set, as introduced in \citet{Arazo2019}.
We use \textsc{SgdM} with momentum 0.9 as the base optimizer, batch size 128, learning rate 0.05, weight decay 1e-3, and cosine learning rate scheduling.
For each level of label noise, we determine the common $\rho$ value using a grid search over SAM in \{0.01, 0.02, 0.05, 0.1, 0.2, 0.5\}.
Then, we search individually for other optimizer-specific hyperparameters to find the best performance.
Hyperparameters to reproduce our results are listed in Tab.~\ref{tab:label-noise-hyper}.
We present the results of the robustness test in Tab.~\ref{tab:label-noise}. WSAM generally achieves better robustness than SAM, ESAM, and GSAM.

\begin{table*}[!htbp]
	\caption{Hyperparameters to reproduce experimental results for ResNet18 on Cifar10 with different noise levels.}
	\medskip
	\label{tab:label-noise-hyper}
	\centering
	\begin{tabular}{ccccc}
		\toprule
		noise level (\%) & 20                                & 40                                & 60                                & 80                                 \\
		\midrule
		Vanilla          & -                                 & -                                 & -                                 & -                                  \\
		SAM              & $\rho=0.2$                        & $\rho=0.2$                        & $\rho=0.1$                        & $\rho=0.05$                        \\
		ESAM             & $\rho=0.2, \beta=0.5, \gamma=0.6$ & $\rho=0.2, \beta=0.5, \gamma=0.6$ & $\rho=0.1, \beta=0.6, \gamma=0.6$ & $\rho=0.05, \beta=0.5, \gamma=0.5$ \\
		GSAM             & $\rho=0.2, \alpha=0.01$           & $\rho=0.2, \alpha=0.02$           & $\rho=0.1, \alpha=0.3$            & $\rho=0.05, \alpha=0.3$            \\
		WSAM             & $\rho=0.2, \gamma=0.91$           & $\rho=0.2, \gamma=0.91$           & $\rho=0.1, \gamma=0.93$           & $\rho=0.05, \gamma=0.92$           \\
		\bottomrule
	\end{tabular}
\end{table*}

\begin{table*}[!htbp]
	\caption{Test of label noise. Top-1 accuracy (\%) for ResNet18 on Cifar10 with different noise levels.}
	\medskip
	\label{tab:label-noise}
	\centering
	\begin{tabular}{lcccc}
		\toprule
		noise level (\%) & 20                        & 40                        & 60                        & 80                        \\
		\midrule
		Vanilla                 & $88.06 \pm 0.48$          & $84.11 \pm 0.39$          & $79.15 \pm 0.43$          & $69.07 \pm 0.95$          \\
		SAM                     & $94.99 \pm 0.09$          & $93.28 \pm 0.16$          & $88.32 \pm 0.28$          & $77.57 \pm 0.51$          \\
		ESAM                    & $94.93 \pm 0.18$          & $92.69 \pm 0.48$          & $86.42 \pm 0.30$          & $32.29 \pm 4.67$          \\
		GSAM                    & $95.11 \pm 0.11$          & $93.25 \pm 0.12$          & $89.90 \pm 0.18$          & $\mathbf{79.09 \pm 0.91}$ \\
		WSAM                    & $\mathbf{95.18 \pm 0.12}$ & $\mathbf{93.33 \pm 0.11}$ & $\mathbf{89.95 \pm 0.12}$ & $78.30 \pm 0.92$          \\
		\bottomrule
	\end{tabular}
\end{table*}

\subsection{Effect of Geometric Structures of Exploration Region}
SAM-like optimizers can be combined with techniques like ASAM and Fisher SAM to shape the exploration neighborhood adaptively.
We conduct experiments with WRN-28-10 on Cifar10, and
compare the performance of SAM and WSAM using the adaptive and Fisher information methods, respectively, to understand how geometric structures of exploration region would affect the performance of SAM-like optimizers.

For parameters other than $\rho$ and $\gamma$, we reuse the configuration in Sec.~\ref{subsec:image-scratch}.
From previous studies \cite{asam, fsam}, $\rho$ is usually larger for ASAM and Fisher SAM. We search for the best $\rho$ in $\{0.1, 0.5, 1.0, \dots, 6.0\}$ and the best $\rho$ is 5.0 in both scenarios. Afterward, we search for the optimal $\gamma$ for WSAM from 0.80 to 0.94 with a step size of 0.02. The best $\gamma$ is 0.88 for both methods.

Surprisingly, the vanilla WSAM is found to be superior across the candidates, as seen in Tab.~\ref{tab:geo-region-res}. It is also worth noting that, contrary to what is reported in \citet{fsam}, the Fisher method reduces the accuracy and causes a high variance.
Therefore, it is recommended to use the vanilla WSAM with a fixed $\rho$ for stability and performance.

\begin{table}[!htbp]
	\caption{Top-1 error (\%) for WRN-28-10 trained on Cifar10.}
	\medskip
	\label{tab:geo-region-res}
	\centering
	\begin{tabular}{cccc}
		\toprule
		     & Vanilla                 & +Adaptive      & +Fisher        \\
		\midrule
		SAM  & $2.94\pm 0.08$          & $2.84\pm 0.04$ & $3.00\pm 0.13$ \\
		WSAM & $\mathbf{2.74\pm 0.07}$ & $2.90\pm 0.08$ & $3.45\pm 0.35$ \\
		\bottomrule
	\end{tabular}
\end{table}

\subsection{Ablation Study}
\label{sec:6.5}
In this section, we conduct an ablation study to gain a deeper understanding of the importance of the \text{``weight decouple''} technique in WSAM. As described in Section~
\ref{sec:4.1}, we compare a variant of WSAM without weight decouple, Coupled-WSAM (outlined in Algorithm~\ref{alg:coupled-WSAM}), to the original method.

\begin{algorithm}[htbp]
	\caption{Generic framework of Coupled-WSAM}
	\label{alg:coupled-WSAM}
	\begin{algorithmic}[1]
		\STATE {\bfseries Input:} parameters $\rho,\epsilon > 0$, $\gamma\in[0,1)$,
		$\bm{w}_1 \in \mathbb{R}^n$, step size $\{\alpha_t\}_{t=1}^T$
		\FOR{$t=1$ {\bfseries to} $T$}
		\STATE $\tilde{\bm{g}}_t = \nabla \ell_t(\bm{w}_t)$
		\STATE $\bm{\delta}_t = \rho\tilde{\bm{g}}_t / (\|\tilde{\bm{g}}_t\| + \epsilon)$
		\STATE $\bm{g}_t = \nabla \ell_t(\bm{w}_t + \bm{\delta}_t)$
		\STATE $\bm{h}_t = \frac{\gamma}{1-\gamma}\bm{g}_t + \frac{1 - 2\gamma}{1 - \gamma}\tilde{\bm{g}}_t$
		\STATE $\bm{m}_t = \phi_t(\bm{h}_1, \dots, \bm{h}_t)$ and $B_t = \psi_t(\bm{h}_1, \dots, \bm{h}_t)$
		\STATE $\bm{w}_{t+1} = \bm{w}_t - \alpha_t B_t^{-1}\bm{m}_t$
		\ENDFOR
	\end{algorithmic}
\end{algorithm}

The results are reported in Tab.~\ref{tab:ablation-study}.
Coupled-WSAM yields better results than SAM in most cases, and WSAM further improves performance in most cases,
demonstrating that the \text{``weight decouple''} technique is both effective and necessary for WSAM.

\begin{table}[!htbp]
	\caption{The ablation study of WSAM for ResNet18, WRN-28-10 on Cifar10 and Cifar100.}
	\medskip
	\label{tab:ablation-study}
	\centering
	\begin{tabular}{cccc}
		\toprule
		                         &              & Cifar10                    & Cifar100                    \\
		\midrule
		\multirow{4}*{ResNet18}  & SAM          & $ 3.68 \pm 0.06 $          & $ 19.97 \pm 0.30 $          \\
		                         & Coupled-WSAM & $ \mathbf{3.58 \pm 0.12} $ & $ 19.49 \pm 0.13 $          \\
		                         & WSAM         & $ 3.62 \pm 0.10 $          & $ \mathbf{19.42 \pm 0.12} $ \\
		\midrule
		\multirow{4}*{WRN-28-10} & SAM          & $ 2.94 \pm 0.08 $          & $ 16.49 \pm 0.09 $          \\
		                         & Coupled-WSAM & $ 2.91 \pm 0.06 $          & $ 16.49 \pm 0.09 $          \\
		                         & WSAM         & $ \mathbf{2.74 \pm 0.07} $ & $ \mathbf{16.33 \pm 0.26} $ \\
		\bottomrule
	\end{tabular}
\end{table}

\subsection{Minima Analysis}
Here, we further deepen our understanding of the WSAM optimizer by comparing the differences in the minima found by the WSAM and SAM optimizers. 
The sharpness at the minima can be described by the dominant eigenvalue of the Hessian matrix. The larger the eigenvalue, the greater the sharpness. This metric is often used in other literature \cite{sam,gsam,Kaddour2022}. 
We use the Power Iteration algorithm to calculate this maximum eigenvalue, a practical tool seen in \citet{hessianeigenthings}.

Tab.~\ref{tab:minima-analysis} shows the differences in the minima found by the SAM and WSAM optimizers.
We find that the minima found by the vanilla optimizer have smaller loss but greater sharpness, whereas the minima found by SAM have larger loss but smaller sharpness, thereby improving generalization.
Interestingly, the minima found by WSAM not only have much smaller loss than SAM but also have sharpness that is close to SAM. 
This indicates that WSAM prioritizes ensuring a smaller loss while minimizing sharpness in the process of finding minima.
Here, we present this surprising discovery, and further detailed research is left for future investigation.

\begin{table}[!htbp]
	\caption{Differences in the minima found by different optimizers for ResNet18 on Cifar10. $\lambda_{max}$ is the dominant Hessian eigenvalue.}
	\medskip
	\label{tab:minima-analysis}
	\centering
	\begin{tabular}{cccc}
		\toprule
		optimizer & loss & accuracy & $\lambda_{max}$ \\
		\midrule
		Vanilla(SGDM) & 0.0026 & 0.9575 & 62.58 \\
		SAM & 0.0339 & 0.9622 & 22.67 \\
		WSAM & 0.0089 & 0.9654 & 23.97 \\
		\bottomrule
	\end{tabular}
\end{table}

\subsection{Hyperparameter Sensitivity}
Compared to SAM, WSAM has an additional hyperparameter $\gamma$ that scales the size of the sharpness term.
Here we test the sensitivity of WSAM's performance to this hyperparameter.
We train ResNet18 and WRN-28-10 models on Cifar10 and Cifar100 with WSAM using a wide range of $\gamma$.
Results in Fig.~\ref{fig:wsam-sensitivity} show that WSAM is not sensitive to the choice of hyperparameter $\gamma$. We also find that the best performance of WSAM occurs almost always in the range between 0.8 and 0.95.

\begin{figure}[!htbp]
	\centering
	\includegraphics[width=0.8\linewidth]{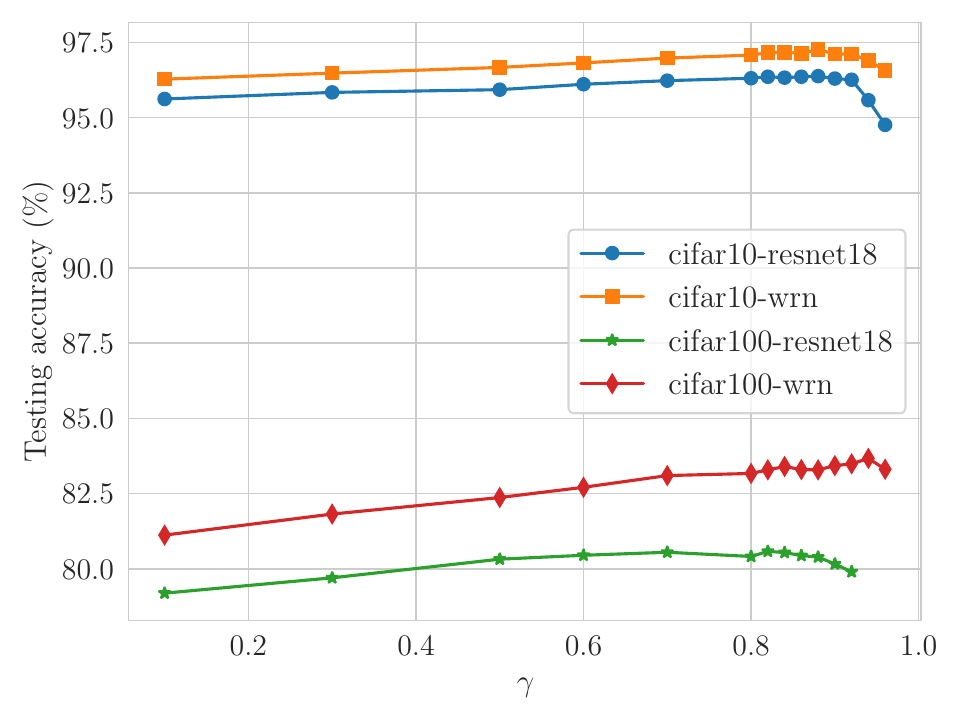}
	\caption{The sensitivity of WSAM's performance to the choice of $\gamma$.}
	\label{fig:wsam-sensitivity}
\end{figure}

	\section{Conclusion}
\label{sec:7}
In this paper, we revisit the structure of SAM, introducing a novel optimizer, called WSAM, which treats the sharpness as a regularization term, allowing for different weights for different tasks.
Additionally, the \text{``weight decouple''} technique is employed to further enhance the performance.
We prove the convergence rate in both convex and non-convex stochastic settings, and derive a generalization bound by combining PAC and Bayes-PAC techniques. Extensive empirical evaluations are performed on several datasets from distinct tasks. The results clearly demonstrate the advantages of WSAM in achieving better generalization.

	\section*{Acknowledgement}
	We thank {\bf Junping Zhao}  and {\bf Shouren Zhao} for their support in providing us with GPU resources.

	\bibliographystyle{ACM-Reference-Format}
	\balance
	\bibliography{reference}

\end{document}